\documentclass[11pt, a4paper, external]{deepmind}
\usepackage[utf8]{inputenc}
\usepackage{amsfonts}
\usepackage{amsmath}
\usepackage{amsthm}
\usepackage{mathtools}
\usepackage{xcolor}
\usepackage{algorithm,algorithmic}
\usepackage{thm-restate}

\newcommand{\Expect}{\mathbb{E}}
\newcommand{\reals}{\mathbb{R}}
\newcommand{\Ac}{\mathcal{A}}
\newcommand{\Sc}{\mathcal{S}}

\newcommand{\Prob}{\mathbb{P}}
\newcommand{\argmin}{\mathop{\rm argmin}}
\newcommand{\argmax}{\mathop{\rm argmax}}
\newcommand{\relint}{\mathop{\rm relint}}

\newcommand{\ie}{\emph{i.e.}}

\usepackage{xspace}
\DeclarePairedDelimiterX{\infdivx}[2]{(}{)}{%
  #1\;\delimsize|\delimsize|\;#2%
}
\newcommand{\kld}[2]{\ensuremath{\mathop{KL}\infdivx{#1}{#2}}\xspace}

\title{On the connection between Bregman divergence and value in regularized Markov decision processes}

\correspondingauthor{bodonoghue@deepmind.com}

\author[1]{Brendan O'Donoghue\\DeepMind}

\begin{abstract}
In this short note we derive a relationship between the Bregman divergence from the current policy to the optimal policy and the suboptimality of the current value function in a regularized Markov decision process. This result has implications for multi-task reinforcement learning, offline reinforcement learning, and regret analysis under function approximation, among others.
\end{abstract}

\begin{document}

\maketitle

We consider a finite state-action discounted MDP given by the tuple $\mathcal{M} = (\Sc, \Ac, P, r, \rho, \gamma)$, where $\Sc$ is the set of states and $\Ac$ is the set of possible actions, $P: \Sc \times \Ac \rightarrow \Delta(\Sc)$ denotes the state transition kernel, $r: \Sc \times \Ac \rightarrow \reals$ denotes the reward function, $\rho \in \Delta(\Sc)$ is the initial state distribution, and $\gamma \in [0,1)$ is the discount factor. The main result of this manuscript holds more generally, but for brevity we shall restrict ourselves to this case. A policy $\pi \in \Delta(\Ac)^{|\Sc|}$ is a distribution over actions for each state, and we shall denote the probability of action $a$ in state $s$ as $\pi(s, a)$. For any policy $\pi \in \Delta(\Ac)^{|\Sc|}$ and continuously-differentiable strictly-convex regularizer $\Omega : \Delta(\Ac) \rightarrow \reals$ we define value functions for each $(s,a) \in \Sc \times \Ac$ as
\[
Q_\Omega^\pi(s,a) =  r(s,a) + \gamma \sum_{s^\prime} P(s^\prime\mid s,a) V_\Omega^\pi(s^\prime)
, \quad
V_\Omega^\pi(s) = \sum_a \pi(s,a) Q_\Omega^\pi(s,a) -\Omega(\pi(s, \cdot)).
\]
The optimal value functions for each $(s,a) \in \Sc \times \Ac$ are given by
\begin{equation}
\label{e-opt}
Q_\Omega^\star(s,a) = r(s,a) + \gamma \sum_{s^\prime}  P(s^\prime\mid s,a) V_\Omega^\star(s^\prime), \quad
V_\Omega^\star(s) = \max_{\pi(s, \cdot) \in \Delta(\Ac)}\left( \sum_a \pi(s,a) Q_\Omega^\star(s,a) - \Omega(\pi(s, \cdot))\right).
\end{equation}
We denote by $\pi_\Omega^\star \in \Delta(\Ac)^{|\Sc|}$ the policy that achieves the maximum for all $s \in \Sc$, which is always attained due to the properties of $\Omega$.
It must satisfy the first-order optimality conditions for the maximum, which is the following inclusion
\begin{equation}
\label{e-pol-inc}
     Q^\star(s, \cdot) - \nabla \Omega (\pi_\Omega^\star(s, \cdot)) \in  N_{\Delta(\Ac)}(\pi_\Omega^\star(s, \cdot)), \quad s \in \Sc,
\end{equation}
where $N_{\Delta(\Ac)}$ is the normal cone of simplex $\Delta(\Ac)$ \cite{ryu2016primer}, \ie,
\begin{equation}
    \label{e-normal-cone}
    N_{\Delta(\Ac)}(\pi) = \{ y \mid y^\top (\pi^\prime -  \pi) \leq 0, \forall \pi^\prime \in \Delta(\Ac)\} \rm{\ if\ } \pi \in \Delta(\Ac) \rm{\ otherwise\ } \emptyset.
\end{equation}
Any policy $\pi$ induces a discounted stationary distribution over states denoted $\mu^\pi \in \Delta(\Sc)$ satisfying
\begin{equation}
\label{e-mu} 
\mu^\pi(s^\prime) = (1-\gamma) \rho(s^\prime) + \gamma \sum_{(s, a)} P(s^\prime \mid s, a) \pi(s, a) \mu^\pi(s) ,\quad s^\prime \in \Sc.
\end{equation}
Finally, the \emph{Bregman divergence} generated by $\Omega$ between two points $\pi, \pi^\prime \in \Delta(\Ac)$ is defined as
\begin{equation}
\label{e-bregman}
D_\Omega(\pi, \pi^\prime) = \Omega(\pi) - \Omega(\pi^\prime) - \nabla \Omega(\pi^\prime)^\top(\pi-\pi^\prime).
\end{equation}
To prove our main result we require a slight generalization of the performance difference lemma (PDL)\cite{kakade2002approximately} to cover the regularized MDP case.
\begin{restatable}{lem}{lempdl}
\label{l-pdl}
For any two policies $\pi$, $\pi^\prime \in \Delta(\Ac)^{|\Sc|}$
\[
(1-\gamma)\Expect_{s \sim \rho}( V_\Omega^{\pi}(s) - V_\Omega^{\pi^\prime}(s))
= \Expect_{s \sim \mu^{\pi}} \left(\sum_a \pi(s,a) Q_\Omega^{\pi^\prime}(s,a) - V_\Omega^{\pi^\prime}(s) - \Omega(\pi(s, \cdot))\right).
\]
\end{restatable}
The proof of this identity is included in the appendix for completeness. For our stronger result we require the following short technical lemma.
\begin{restatable}{lem}{lrelint}
\label{l-relint}
If $\pi^\star \in \relint(\Delta(\Ac))$ then $y^\top (\pi - \pi^\star)=0$ for all $y \in N_{\Delta(\Ac)}(\pi^\star)$, $\pi \in \Delta(\Ac)$.
\end{restatable}
\begin{proof}
Given $\pi^\star \in \relint(\Delta(\Ac))$ assume that we can find $\pi \in \Delta(\Ac)$ such that for some  $y \in N_{\Delta(\Ac)}(\pi^\star)$ we have $y^\top (\pi - \pi^\star) < 0$. Let $\Delta = \pi - \pi^\star$. For sufficiently small $\epsilon > 0$ there exists policy $\pi^{\prime} = \pi^\star - \epsilon \Delta$ that is in $\Delta(\Ac)$ since $\pi^\star$ is in the relative interior. By assumption $y^\top \Delta = y^\top (\pi - \pi^\star) < 0$. However, from the normal cone property $0 \leq -y^\top (\pi^\prime- \pi^\star) = \epsilon y^\top \Delta$, which is a contradiction.
\end{proof}

With these we are ready to present the main result of this note.
\begin{restatable}{thm}{main}
\label{t-main}
For any policy $\pi \in \Delta(\Ac)^{|\Sc|}$
\[
\Expect_{s \sim \mu^\pi}D_\Omega(\pi(s, \cdot), \pi_\Omega^\star(s, \cdot)) \leq (1-\gamma)\Expect_{s \sim \rho}( V_\Omega^\star(s) - V_\Omega^\pi(s)),
\]
moreover, if $\pi_\Omega^\star \in \relint(\Delta(\Ac))$ then
\[
\Expect_{s \sim \mu^\pi}D_\Omega(\pi(s, \cdot), \pi_\Omega^\star(s, \cdot)) = (1-\gamma)\Expect_{s \sim \rho}( V_\Omega^\star(s) - V_\Omega^\pi(s)).
\]
\end{restatable}
\begin{proof}
Let $y(s) = Q^\star(s, \cdot) - \nabla \Omega(\pi_\Omega^\star(s, \cdot))$ for $s \in \Sc$, and note that $y(s) \in N_{\Delta(\Ac)}(\pi_\Omega^\star(s, \cdot))$ from \eqref{e-pol-inc}, then
\begin{align*}
    \Expect_{s \sim \mu^\pi}D_\Omega(\pi(s, \cdot), \pi_\Omega^\star(s, \cdot)) 
    &= \Expect_{s \sim \mu^\pi} \left(\Omega(\pi(s, \cdot)) - \Omega(\pi_\Omega^\star(s, \cdot)) - (Q^\star(s, \cdot) - y(s)))^\top (\pi(s, \cdot) - \pi_\Omega^\star(s, \cdot) )\right)\\
    &\leq \Expect_{s \sim \mu^\pi} \left(\Omega(\pi(s, \cdot)) - \Omega(\pi_\Omega^\star(s, \cdot)) - Q^\star(s, \cdot)^\top (\pi(s, \cdot) - \pi_\Omega^\star(s, \cdot) )\right)\\
    &= \Expect_{s \sim \mu^\pi} \left(V_\Omega^\star(s) - \sum_a \pi(s,a) Q_\Omega^\star(s,a) + \Omega(\pi(s, \cdot))\right)\\
    &=(1-\gamma)\Expect_{s \sim \rho}( V_\Omega^{\star}(s) - V_\Omega^\pi(s)),
\end{align*}
where the first line replaces $\nabla \Omega(\pi_\Omega^\star(s, \cdot))$ in the definition of the Bregman divergence \eqref{e-bregman}, the second uses the normal cone property \eqref{e-normal-cone}, the third line substitutes in the definition of $V_\Omega^\star$ from \eqref{e-opt}, 
and the final line uses the regularized PDL. To get the stronger statement, we use Lemma \ref{l-relint} which replaces the inequality in the second line with an equality.
\end{proof}

Since the KL-divergence is the Bregman divergence generated by the negative-entropy function, and since the entropy-regularized optimal policy is always in the relative interior of the simplex for bounded rewards, we have the following corollary.
\begin{restatable}{cor}{kl}
Let $\tau > 0$ be a regularization parameter and denote entropy by $H$ and KL-divergence by $\kld{\cdot}{\cdot}$ and set $\Omega = -\tau H$, then for any policy $\pi \in \Delta(\Ac)^{|\Sc|}$ we have
\[
\Expect_{s \sim \mu^\pi} \kld{\pi(s, \cdot)}{\pi_{\Omega}^\star(s, \cdot)} = \frac{(1-\gamma)}{\tau}\Expect_{s \sim \rho} (V_{\Omega}^\star(s) - V_{\Omega}^\pi(s)).
\]
\end{restatable}
This identity and related ones have previously been used to bound policy performance between an entropy-regularized optimistic policy and Thompson sampling \cite{o2020making} and to prove the linear convergence of entropy-regularized policy gradient methods \cite{mei2020global}.

\paragraph{Acknowledgements.} I thank Tor Lattimore for spotting an error in a earlier draft of this note and Ted Moskovitz for useful discussions.
\bibliographystyle{abbrv}
\bibliography{refs}
 
\section{Proof of regularized PDL}
The following lemma will be useful to prove the PDL. First, let us define $P^\pi \in \mathbf{R}^{|\Sc| \times |\Sc|}$ to be the state transition probability matrix under policy $\pi$, \ie, $P^\pi_{s s^\prime} = \sum_a P(s^\prime \mid s, a) \pi(s, a) = \Prob(s \rightarrow s^\prime \rm{\ under\ policy \ } \pi)$.
\begin{restatable}{lem}{basic}
\label{l-basic}
For any vector $x \in \mathbf{R}^{|\Sc|}$
\[
(\mu^\pi)^\top (I - \gamma P^\pi) x = (1-\gamma) \rho^\top x.
\]
\end{restatable}
\begin{proof}
Using vector notation, we can write \eqref{e-mu} as
$\mu^\pi = (1-\gamma) \rho + \gamma (P^\pi)^\top \mu^\pi = (1-\gamma) (I - \gamma (P^\pi)^\top)^{-1} \rho$.
The result follows from
$(\mu^\pi)^\top (I - \gamma P^\pi)x =  (1-\gamma) \rho^\top(I - \gamma P^\pi)^{-1} (I - \gamma P^\pi) x = (1-\gamma) \rho^\top x$.
\end{proof}
\lempdl*
\begin{proof}
Define vectors $r^\pi, \Omega^\pi, q^{(\pi , \pi^\prime)} \in \reals^{|\Sc|}$ as
\[
r^\pi_s = \sum_a \pi(s,a) r(s,a), \quad \Omega^\pi_s = \Omega(\pi(s, \cdot)), \quad q^{(\pi , \pi^\prime)}_s = \sum_a \pi(s,a) Q_\Omega^{\pi^\prime}(s,a),
\]
and note that $q^{(\pi , \pi^\prime)}$ and $V_\Omega^\pi$ satisfy
\[
q^{(\pi , \pi^\prime)} = r^\pi + \gamma P^\pi V^{\pi^\prime}_\Omega, \quad V_\Omega^\pi = r^\pi - \Omega^\pi + \gamma P^\pi V_\Omega^\pi.
\]
Then, using Lemma \ref{l-basic} we have
\begin{align*}
(1-\gamma) \rho^\top( V_\Omega^{\pi} - V_\Omega^{\pi^\prime}) &= (\mu^\pi)^\top \left(V_\Omega^{\pi} - V_\Omega^{\pi^\prime} - \gamma P^\pi (V_\Omega^{\pi} - V_\Omega^{\pi^\prime}) \right)\\
&= (\mu^\pi)^\top \left(r^\pi - \Omega^\pi - V_\Omega^{\pi^\prime} + \gamma P^\pi V_\Omega^{\pi^\prime} \right)\\
&= (\mu^\pi)^\top \left(q^{(\pi , \pi^\prime)} - \Omega^\pi - V_\Omega^{\pi^\prime} \right).
\end{align*}
\end{proof}
\end{document}